\documentclass[letter]{article}

\DeclareUnicodeCharacter{025B}{\ensuremath{\varepsilon}}

\usepackage[utf8]{inputenc} 
\usepackage[T1]{fontenc}    
\usepackage{url}            
\usepackage{booktabs}       
\usepackage{amsfonts}       
\usepackage{nicefrac}       
\usepackage{microtype}      
\usepackage[dvipsnames]{xcolor}         
\usepackage[bookmarks=false,colorlinks=true,citecolor=NavyBlue,linkcolor=NavyBlue]{hyperref}
\usepackage{appendix}

\usepackage{amsmath}
\usepackage{amsfonts}
\usepackage{amssymb}
\usepackage{nicefrac}
\usepackage{tocloft}
\usepackage{etoolbox}
\usepackage{enumitem}
\usepackage{indentfirst}
\usepackage{caption}
\usepackage{lmodern}
\usepackage{nicefrac}
\usepackage{fancyhdr}
\usepackage{algorithm, algorithmic}
\usepackage{amsthm}
\usepackage{subfig}
\usepackage{geometry}
  \newgeometry{
    textheight=9in,
    textwidth=5.64in,
    top=1in,
    headheight=12pt,
    headsep=25pt,
    footskip=30pt
  }

\newtheorem{theorem}{Theorem}[section]
\newtheorem{corollary}{Corollary}[theorem]
\newtheorem{lemma}[theorem]{Lemma}

\newtheorem{definition}{Definition}[section]

\makeatletter
\newcommand\fs@betterruled{%
  \def\@fs@cfont{\bfseries}\let\@fs@capt\floatc@ruled
  \def\@fs@pre{\vspace*{5pt}\hrule height.8pt depth0pt \kern2pt}%
  \def\@fs@post{\kern2pt\hrule\relax}%
  \def\@fs@mid{\kern2pt\hrule\kern2pt}%
  \let\@fs@iftopcapt\iftrue}
\makeatother

\def\argmax{\operatornamewithlimits{arg\,max}}

\newcommand{\QX}{Q_{\scriptscriptstyle X}}

\newcommand{\BE}{\mathbb{E}}

\newcommand{\BA}{\mathbb{A}}

\newcommand{\BI}{\mathbb{I}}

\newcommand{\muM}{\widehat{\mu}}

\def \CD{\mathcal{D}}

\usepackage{mathtools}
\DeclarePairedDelimiter\ceil{\lceil}{\rceil}

\DeclareMathOperator {\sign}{sign}

\DeclareMathOperator {\med}{med}

\newcommand{\RR}{\mathbb{R}}

\newcommand{\NN}{\mathbb{N}}

\newcommand{\EE}{\mathbb{E}}
\newcommand{\PP}{\mathbb{P}}

\usepackage{amsthm}

\DeclareFontFamily{U}{mathx}{\hyphenchar\font45}
\DeclareFontShape{U}{mathx}{m}{n}{<-> mathx10}{}
\DeclareSymbolFont{mathx}{U}{mathx}{m}{n}
\DeclareMathAccent{\widebar}{0}{mathx}{"73}

\newtheorem{assumption}{A\hspace{-1.4mm}}
\newtheorem{innercustomgeneric}{\customgenericname}
\providecommand{\customgenericname}{}
\newcommand{\newcustomtheorem}[2]{%
	\newenvironment{#1}[1]
	{%
		\renewcommand\customgenericname{#2}%
		\renewcommand\theinnercustomgeneric{##1}%
		\innercustomgeneric
	}
	{\endinnercustomgeneric}
}

\newcustomtheorem{customthm}{Theorem}
\newcustomtheorem{customlemma}{Lemma}

\usepackage[american]{babel}

\usepackage{natbib} 
\usepackage{mathtools} 
\usepackage{booktabs} 
\usepackage{tikz}
\allowdisplaybreaks

\title{\Large\textsc{Data-Driven Upper Confidence Bounds with Near-Optimal Regret for Heavy-Tailed Bandits}}

\author{%
   {Ambrus Tam{\'a}s\qquad\quad Szabolcs Szentp\'eteri\qquad\quad Bal\'azs Csan\'ad Cs\'aji}\\[5mm]
   \normalsize Institute for Computer Science and Control (SZTAKI),\\
   \normalsize Hungarian Research Network (HUN-REN), Budapest, Hungary 1111\\[2mm]
   \normalsize Department of Probability Theory and Statistics, Institute of Mathematics,\\
   \normalsize Faculty of Science, E{\"o}tv{\"o}s Lor{\'a}nd University (ELTE), Budapest, Hungary 1117\\[2mm]
  {\tt\normalsize \{ambrus.tamas, szabolcs.szentpeteri, balazs.csaji\}@sztaki.hu}
}

\date{}

\begin{document}

\maketitle

\begin{abstract}
  Stochastic multi-armed bandits (MABs) provide a fundamental reinforcement learning model to study sequential decision making in uncertain environments. The upper confidence bounds (UCB) algorithm gave birth to the renaissance of bandit algorithms, as it achieves near-optimal regret rates under various moment assumptions. Up until recently most UCB methods relied on concentration inequalities leading to confidence bounds which depend on moment parameters, such as the variance proxy, that are usually unknown in practice. In this paper, we propose a new distribution-free, data-driven UCB algorithm for symmetric reward distributions, which needs no moment information. The key idea is to combine a refined, one-sided version of the recently developed resampled median-of-means (RMM) method with UCB. We prove a near-optimal regret bound for the proposed anytime, parameter-free RMM-UCB method, even for heavy-tailed distributions.
\end{abstract}

\section{Introduction}

In this paper, we study {\em stochastic multi-armed bandit} (MAB) problems which are special online learning models \citep{lattimore2020bandit}. They are fundamental to study the notorious {\em exploration-exploitation dilemma} of reinforcement learning, and also have a wide range of direct applications. They were initially introduced to study sequential clinical trials, but recent applications include recommender systems, portfolio optimization, adaptive routing, clustering, anomaly detection and Monte Carlo tree search \citep{Bouneffouf2020}. 

The theory of MABs has a rich history. They were introduced by \citet{robbins1952some}, nevertheless, the modern renaissance of bandit algorithms started with the publication of the {\em upper confidence bounds} (UCB) method by \citet{auer2002finite}. Since then, there has been continued interest in bandit algorithms \citep{bubeck2012regret, lattimore2020bandit}. 

MABs are studied under diverse conditions and have an ever growing number of variants. We focus on stochastic MABs with finitely many arms that is one of the standard setups, which was extensively studied in the past \citep{lattimore2020bandit}. We deal with UCB algorithms which aim to find the right balance between exploration and exploitation based on the optimism principle, by choosing the arm with the highest UCB at each round. A vast amount of research has been done under various assumptions on the moment generating function \citep{bubeck2012regret}. Famously, for subgaussian bandits \citet{auer2002finite} showed that the optimal regret rate is $O(\sqrt{n})$, which can be achieved by UCB type algorithms.

In the heavy-tailed regime, for rewards with bounded central moment of order $1+a$, \citet{bubeck2013bandits} showed that the regret is at least $O(n^{1/(1+a)})$. They also presented a near-optimal robust UCB algorithm, which however uses $a$ as well as the moment bound $M$. On the other hand, these (hyper) parameters are typically {\em unknown} in practice, similarly to the variance proxies of subgaussian rewards. These call for data-driven, {\em parameter-free} methods. \citet{cesa2017boltzmann} removed the dependence on $M$ for $a=1$ by applying a robust estimator with random exploration. In \citep{lee2020optimal} an adaptively perturbed exploration (APE$^2$) scheme was proposed which uses only the value of $a$ and not the value of $M$. \citet{wei2021} developed the robust version of minimax optimal strategy for stochastic bandits (MOSS) and proved optimal regret bounds if $M$ is given for the agent. In \citep{lee2022minimax} the MOSS is combined with a $1+a$-robust estimator to achieve minimax optimality without using the knowledge of $M$, however, their method essentially relies on the knowledge of $a$. The first fully data-driven method is defined in \citep{marsigli2022towards} based on an adaptive trimmed mean estimator and under the tuncated non-positivity assumption near-optimal regret bounds are proved. \citep{kagrecha2019distribution} developed a completely parameter-free (distribution oblivious) method for a risk-aware best arm identification problem.

\subsection{Stochastic Multi-Armed Bandits}

A finite-armed stochastic bandit consists of an environment and an agent. The environment is described by the set of arms $[K] \doteq \{1, \dots, K\}$ 
and the set of reward distributions $\{ \nu_1, \dots, \nu_K\}$, which corresponds to the arms. At each round $t \in \NN$ the agent picks an arm $I_t$ from action set $[K]$ and receives a random reward $X_t$ drawn independently from distribution $\nu_{I_t}$. The goal of the agent is to minimize the (expected) regret
\begin{equation}
    \label{def:regret}
    R_n \doteq n \max_{i \in [K]} \mu_i - \EE \bigg[ \sum_{t=1}^n X_t \bigg],
\end{equation}
or equivalently to maximize the gained rewards. The main challenge of the problem is that the distributions are unknown for the agent, henceforth the actions need to be chosen based on previously observed (random) rewards. Each distribution $\nu_i$ has a finite expected value $\mu_i$, and for notational simplicity let $\mu_1 > \mu_i$ for $i \geq 2$ and $\Delta_i \doteq \mu_1 - \mu_i$ for $i \in [K]$.

\subsection{Data-Driven Bandits}

The standard approach to stochastic MABs is to apply an UCB-style algorithm. Most of these methods strongly rely on some posed moment assumptions. The method of \cite{auer2002finite} uses the form of the moment generating function and also the moments of the reward distributions. However, in practice these moment values are typically {\em unknown}. A recent paper of \cite{khorasani2023maximum} proposed a {\em data-driven} approach, which is able to overcome this challenge. They introduced the so-called {\em maximum average randomly sampled} (MARS) algorithm, which does not rely on the moment generating function or on the moments, hence it is parameter-free. They prove theoretical guarantees and achieve competitive performance results for symmetric bandits. MARS achieves a regret bound of
\begin{equation}
    R_n \leq \sum_{i: \Delta_i > 0}  \bigg( 3 + \frac{3\log(n)}{x_i}\bigg) \Delta_i, \qquad
    \text{with}
    \qquad x_i = - \log\bigg( \frac{1}{2} + \frac{1}{2} \exp(- \psi_i^*(\Delta_i))\bigg),
\end{equation}
on a {\em fixed horizon}, where $\psi_i^*(z) \doteq \sup_{\lambda\in \RR} (\lambda z - \psi_i(\lambda))$ is the Legendre-Fenchel transform of the moment generating function $\psi_i$ of distribution $\nu_i$.

\subsection{Heavy-Tailed Bandits}

If the reward distributions are heavy-tailed, then providing UCBs for the means and proving regret bounds are challenging. \cite{bubeck2013bandits} do not assume that the moment generating function exists for every $\nu_i$, they pose the much milder moment assumption of 
\begin{equation}
    \EE_{\zeta_i \sim \nu_i} [|\zeta_i - \EE \zeta_i|^{1+a}] \leq M < \infty,
\end{equation}
with some known $a > 0$ and $M$. Then, for a heavy-tailed UCB algorithm they prove that 
\begin{equation}
\label{eq:rate}
\begin{aligned}
    R_n \leq \sum_{i: \Delta_i > 0}  \bigg( 2c\bigg(\frac{M}{\Delta_i}\bigg)^{1/a} \log(n) + 5\Delta_i\bigg),
\end{aligned}
\end{equation}
for some constant $c$. This implies a rate of $n^{1/(1+a)}$ for $R_n$. \cite{bubeck2013bandits} also prove that there exists an environment of $K$ distributions such that for every strategy the regret is
\begin{equation}
    R_n \geq \text{const}\cdot K^{a/(1+a)} n^{1/(1+ a)}.
\end{equation}
A limitation of their methods is that the agent needs to know $a$ and $M$ to construct the UCBs at each round.
In this paper we propose a {\em parameter-free}, anytime, heavy-tailed UCB algorithm which achieves the same optimal rate as \eqref{eq:rate} up to a logarithmic term, {\em without knowing $a$ and $M$}. 
We combine the advantages of the recently presented {\em resampled median-of-means} (RMM) estimator \citep{tamas2024data} with UCB to prove similar rates as in \citep{bubeck2012regret} on a data-dependent manner as in \citep{khorasani2023maximum}.
Moreover, MARS can be seen as a special case of the proposed RMM-UCB method.

\section{Median-of-Means Type Upper Confidence Bounds}

The fundamental ingredient of an UCB algorithm is the confidence bound construction. In this section we present a new, {\em one-sided} version of the recently developed resampled median-of-means (RMM) method \citep{tamas2024data} to construct exact UCBs for the mean of a symmetric, heavy-tailed distribution. An i.i.d.\ sample $\CD_0 \doteq \{X_i\}_{i=1}^n$ is given from an unknown distribution $\QX$ which is {\em symmetric} about an unknown parameter $\mu$, that is 
\begin{assumption}\label{ass:iid}
        {\em $X_1, \dots, X_n$ are i.i.d.}
    \end{assumption}
    \begin{assumption}\label{ass:symmetry}
        {\em $\QX$ is symmetric about $\mu$}
\end{assumption}
We aim at constructing non-asymptotically valid upper confidence bounds for $\mu$. 
We use no further assumptions other than symmetry and manage to reach any user-chosen (rational) confidence level.
We also prove non-asymptotic, probably approximately correct (PAC) bounds for the distances between the UCB and $\mu$ under the extra assumption that 
\begin{assumption}\label{ass:finite_mom}
   {\em $\EE [| X - \EE X |^{1+ a}] = M < \infty,$\, for\, $a \in (0,1]$.
   }
\end{assumption}

Let $\med(X)$ denote a median of random variable $X$. Let $X_{(1)} \leq X_{(2)} \leq \dots \leq X_{(n)}$ denote the {\em ordered sample}, then, the {\em empirical median} of sample $X_1, \dots, X_n$ is defined by 
\begin{equation}
    \med(\CD_0) \doteq \med( X_1, \dots, X_n) \doteq
        \begin{cases}
		  X_{(n/2)}  & \mbox{if $k$ is even,}\\
		  X_{(\lfloor n/2 \rfloor + 1) } & \mbox{if $k$ is odd}.
	  \end{cases}
\end{equation}
If $\QX$ is symmetric,  $\mu = \med(X)$.
In this paper we assume that $\EE[X_1]= \mu$, however, most of the presented results regarding the confidence level holds without this moment assumption.

\subsection{Remarks on the Empirical Mean}
Let us denote the well-known empirical mean by $\bar{\mu}_n$.
By the celebrated Gauss-Markov theorem \citep{plackett1949historical} if $\sigma^2 \doteq D^2(X_1)
< \infty$, then $\bar{\mu}_n$ is BLUE: it is the ``best'' linear unbiased estimator, i.e., $\bar{\mu}_n$ has the lowest variance among linear unbiased estimators. By the central limit theorem (CLT) if $\sigma^2 < \infty$, then $\sqrt{n}(\bar{\mu}_n - \mu) \to Z$ in distribution, where $Z$ is a zero mean normal variable with variance $\sigma^2$.
This convergence provides gaussian-type asymptotic bounds, but for bandits we typically seek finite sample guarantees.
If $X$ is $\sigma$-subgaussian, i.e., $\EE[ \exp( \lambda(X-\mu)) ]\leq \exp(\sigma^2\lambda^2/2)$ for all $\lambda \in \RR$, then one has a non-asymptotic subgaussian concentration inequality for the mean. If the distribution is not subgaussian, Chebyshev's inequality provides an exponentially weaker bound w.r.t.\ the confidence parameter. Chebyshev's bound is tight in some sense, i.e., for every sample size $n$ and confidence parameter there exists a distribution with variance $\sigma^2$ on which the mean performs poorly, therefore the empirical mean is not subgaussian for $\mu$ in the general case. 

\subsection{Median-of-Means}
It is a somewhat surprising result that one can construct subgaussian estimators for the mean of heavy-tailed distribution without assuming subgaussianity for the distribution of the sample. One of the most well-known estimator of this kind is the so-called median-of-means (MoM) estimator introduced by \citet{nemirovskij1983problem}. 
Let $k$ be an integer smaller than $n$. For the MoM estimate one needs to partition the dataset into $k$ groups of almost the same size, i.e., let $\tilde{n} = \lfloor n/k \rfloor$ and for $\ell = 1, \dots, k$ if $\tilde{n}\ell + \ell \leq n$ let   $B_\ell \doteq \{ X_{\ell}, \dots, X_{\ell + \tilde{n}k}\}$,
and $B_\ell \doteq \{ X_{\ell}, \dots, X_{\ell + (\tilde{n}-1)k}\}$ otherwise.
Partition $B_1, \dots, B_k$ can be defined in many different way as long as $|B_\ell| \geq \tilde{n}$ holds.
The MoM estimator is defined by
    \begin{equation}
    \begin{aligned}
        &\widehat{\mu}(\CD_0) = \widehat{\mu}(X_1, \dots, X_n) \doteq \med \Bigg(\frac{1}{|B_1|} \sum_{i \in B_1} X_i, \dots , \frac{1}{|B_k|} \sum_{i\in B_k} X_i \Bigg).
    \end{aligned}
    \end{equation}
    The MoM estimator reaches optimal finite sample performance in the following sense \citep{bubeck2013bandits}, \citep{devroye2016sub} and \citep{lugosi2019mean}:
    \begin{theorem}
        Assume A\ref{ass:iid} and A\ref{ass:finite_mom} and let $\EE [X_1] = \mu$,
        then the MoM estimator $\widehat{\mu}$ with $k = \ceil{8 \log(2/\delta)}$ blocks satisfies
        \[ \PP\bigg( \, |\muM - \mu | \leq 8\bigg(\frac{12 M^{1/a} \log( 1 / \delta)}{n} \bigg) ^{a/ (1 + a)} \, \bigg) > 1 - \delta.\]
        Moreover, for any mean estimator $\mu_n$, 
        sample size $n \in \NN$ and $\delta > 0$ there exists a distribution with mean $\mu$ and $(1+a)$th central moment $M$ such that
        \[ \PP\bigg( \, |\mu_n - \mu |  > \bigg(   \frac{M^{1/a} \log( 2 / \delta)}{n} \bigg) ^{a/ (1 + a)} \, \bigg) \geq \delta.\]
    \end{theorem}
    The second part of the theorem provides a lower bound for the rate w.r.t.\ $\delta$ and $n$, hence the MoM estimator achieves an optimal convergence rate up to a constant factor.

    \subsection{One-Sided Resampled Median-of-Means}

    In this section, we introduce the one-sided version of the recently developed resampled MoM (RMM) method \citep{tamas2024data}, which then can be used to construct exact upper confidence bound for the mean of symmetric variables.
    For a given $\theta \in \RR$ let us consider
    \begin{equation}
        H_0: \mu \leq \theta \qquad \text{vs} \qquad H_1: \mu > \theta.
    \end{equation}
    We assume A\ref{ass:iid} and A\ref{ass:symmetry} and construct 
    non-asymptotically exact UCBs based on a rank test. Then, we
    prove exponential PAC-bounds w.r.t.\ $k$ for 
    the distances between the UCBs and the true parameter
    under A\ref{ass:finite_mom}.
    Note that from A\ref{ass:finite_mom} it follows that $\EE X_1 = \mu < \infty$.
    In particular for $a=1$ assumption A\ref{ass:finite_mom} requires $\sigma^2 < \infty$. It is one of the main aim of this paper to deal with cases where $a < 1$, i.e., when the observed distribution is heavy-tailed. We emphasize that the presented test can decide about $H_0$ without using the knowledge of $a$ and $M$. These values are only included in the theoretical analysis of the proposed UCB.
    
    Let $p$ be the desired (rational) significance level.
    Let us find integers $r$ and $m$ such that $p = \nicefrac{r}{m}$. 
    We present a resampling algorithm to test $H_0$ for any $\theta \in \RR$.\
    Let $X(\theta) = \alpha(X - \theta)+ \theta$ be a random variable where $\alpha$ is a random sign, i.e., a Rademacher variable independent of $X$. One can immediately see that $\EE[ X(\theta) ] = \theta$ and $X(\theta)$ is symmetrical about $\theta$.
    Let $W = X - \mu$ be the centered version of $X$ and $\mathcal{W}_0 \doteq \{W_i\}_{i=1}^n$, where $W_i = X_i - \mu$ for $i \in [n]$. Note that $\{W_i\}_{i=1}^n$ are not observed.
    One of our main observations is that $X(\mu)= \mu + \alpha \cdot W$ and $X  = \mu + W$ have the same distribution, because $W$ is symmetric about zero. Furthermore, one can prove that $X$ and $X(\mu)$ are conditionally i.i.d.\ w.r.t.\ $\{|W_i|\}_{i=1}^n$, hence they are also exchangeable \citep{csaji2014sign}. On the other hand, if $\theta \neq \mu$, then the distribution of $X$ differs from the distribution of $X(\theta)$, e.g., $X(\theta)$ is symmetrical about $\theta$ whereas $X$ is symmetrical about $\mu$. We aim to construct UCBs by utilizing this difference.
    Our procedure generates alternative samples and compares them to the original one.
    
     Specifically, 
     let $\{\alpha_{i,j}\}$ be i.i.d.\ Rademacher variables for $i \in [n]$, $j \in [m-1]$ and
    \begin{equation}
        \CD_j(\theta) \doteq \{\, \alpha_{1,j}(X_1 - \theta) + \theta, \dots , \alpha_{n,j}(X_n - \theta) + \theta\,\}
    \end{equation}
    be parameter dependent {\em alternative samples} for $j = 1, \dots, m-1$ and $\CD_0(\theta) \doteq \CD_0$ for $\theta \in \RR$. We can observe that $\CD_j(\theta)$ is an i.i.d.\ sample from the distribution of $X(\theta)$ for $j \neq 0$.
    Let us also define {\em reference variables} using the MoM estimator $\widehat{\mu}(\cdot)$ as
    \begin{equation}
        S_j(\theta) \doteq \widehat{\mu}(\CD_j(\theta))- \theta \qquad \text{for}\qquad j=0,\dots, m-1.
    \end{equation}
    Then, we decide about $H_0$ by comparing $S_0(\theta)$ to $S_j(\theta)$
    for $j \in [m-1]$ with a {\em ranking function} $R$. 
    Notice that if $\theta = \mu$, then each reference variable has the same distribution, while if $\theta \neq \mu$, then $S_0(\theta)$ is farther from $\theta$ than $S_j(\theta)$ with high probability for $j=1, \dots, m-1$.
    \begin{algorithm}[t]
    \caption{One-Sided RMM for the Mean}\vspace{1mm}
    \textbf{Inputs:} i.i.d. sample $\CD_0$, rational significance level $p$, integer MoM parameter $k$\\[-2mm]
    \begin{algorithmic}[1]\label{hypothesis_algo}
         \hrule
         \STATE Choose integers $1 \leq r \leq m$ such that $p = \nicefrac{r}{m}$.\\[1mm]
         \STATE Generate $n(m-1)$ independent Rademacher signs $\{\alpha_{i, j}\}$ for $j \in [m-1]_{\scriptscriptstyle 0}$ and $i \in [n]$. \\[1mm]
         \STATE Generate a random permutation $\pi$ on $[m-1]_{\scriptscriptstyle 0} $ independently from $\CD_0$ and $\{\alpha_{i,j}\}$ uniformly from the symmetrical group.\\[1mm]
         \STATE Construct alternative samples for $j \in [m-1]$ by
         \[\CD_j(\theta)= \{ \alpha_{1,j}(X_1 - \theta) + \theta, \dots , \alpha_{n,j}(X_n - \theta) + \theta\} \]
         and let $\CD_j^\pi \doteq (\CD_j(\theta), \pi(j))$ for $j \in [m-1]_{\scriptscriptstyle 0}$.\\[1mm]
         \STATE Compute the reference variables for $j=0,\dots, m-1$: 
          $S_j(\theta) \doteq \muM(\CD_j(\theta))- \theta$
         \STATE Compute the rank: $R(\theta) =  1 + \sum_{j=1}^{m-1} \BI\,\big(\,S_0(\theta) \prec_\pi S_j(\theta)\,\big)$
         \STATE Reject $H_0$ if and only if\, $R(\theta) > m - r$.
    \end{algorithmic}
    \end{algorithm}
    
    In conclusion, let us construct a ranking function as
    \begin{equation}
    \begin{aligned}
        &R(\theta) 
        \doteq 1 + \sum_{j=1}^{m-1} \BI\big( \,S_0(\theta) \prec_\pi S_j(\theta)\,\big),
    \end{aligned}
    \end{equation}
    where $\prec_\pi$ is defined as the standard $<$ ordering with tie-breaking \citep{csaji2014sign}, i.e.,
    \begin{equation}
    \begin{gathered}
    S_j(\theta)\, \prec_\pi\, S_k(\theta) \Longleftrightarrow
    \Big(S_j(\theta)\, <\, S_k(\theta) \text{ \,or\,} \big(\,S_j(\theta)\,=\,S_k(\theta) \text{ \,and\, } \pi(j) \,<\, \pi(k)\,\big)\Big),
    \end{gathered}
    \end{equation}
    where $\pi$ is a random permutation on $[m-1]_{\scriptscriptstyle 0}$ generated independently from $\CD_0$ and $\{\alpha_{i,j}\}$ uniformly from the symmetric group.
    The one-sided RMM test rejects $\theta$ if and only if $R(\theta) > m-r$. The step by step procedure is presented in Algorithm \ref{hypothesis_algo}. One of our main results is that this test admits an exact confidence level:
    
    \begin{theorem}\label{thm:exact-coverage}
    Assume A\ref{ass:iid} and A\ref{ass:symmetry}, then we have\,
    $\PP\,(\, R( \mu ) > m- r \,) = \frac{r}{m}$.
    \end{theorem}
    The proof can be found in the supplementary material. We emphasize that the confidence level is exact for every $n \in \NN$ and for every symmetrical distribution $\QX$.
    
    \subsection{Resampled Median-of-Means Upper Confidence Bounds}

    Let us use the proposed hypothesis test to construct upper confidence bounds for $\mu$. We build a confidence set out of those parameters that are accepted by Algorithm \ref{hypothesis_algo}.
    For the test we need to generate random signs, however, this procedure does not need to be repeated for every parameter. We use the same set of random signs for every $\theta$. Hence, the confidence region is defined by
    $\Theta_n \doteq \{ \,\theta : R(\theta)  \leq m - r\, \}.$ This lead to the upper confidence bound:
    \begin{equation}
        U \doteq \sup\, \{ \,\theta : R(\theta)  \leq m - r \,\}.
    \end{equation}
    We will show that $\Theta_n$ is either $(-\infty, U)$ or $(-\infty, U]$ depending on the tie-breaking permutation $\pi$. Note that $U$ can be infinite. An important consequence of Theorem \ref{thm:exact-coverage} is as follows:
    \begin{corollary}
    $\Theta_n$ is an {exact confidence region} for $\mu$, i.e.,\, 
    $ \PP ( \mu \in \Theta_n ) = 1- \frac{r}{m}$.
    \end{corollary}
    It can also be shown that the inclusion probability for $\theta \neq \mu$ goes to zero with an exponential rate as the sample size tends to infinity \citep{tamas2024data}. 
    We provide a formula for $U$ by Lemma \ref{lemma:finite-rep}, which can be efficiently computed from the data and thus we present a finite representation for $\Theta_n$. The proof is presented in the supplements. The whole procedure of the UCB construction is described by Algorithm \ref{alg:conf_interval_alg}.

     \begin{algorithm}[t]
    \caption{RMM  Upper Confidence Bound}
    \textbf{Inputs:} i.i.d. sample $\CD_0$, rational significance level $p$,
         integer MoM parameter $k$\\[-2mm]
    \begin{algorithmic}[1]\label{alg:conf_interval_alg}
         \hrule
         \STATE Choose integers $1 \leq r < m$ such that $p = \nicefrac{r}{m}$.\\[1mm]
         \STATE Generate $n(m-1)$ independent Rademacher signs $\{\alpha_{i, j}\}$ for $(i,j) \in [n] \times [m-1]_0$. \\[1mm]
         \STATE Generate a random permutation $\pi$ on $[m-1]_{\scriptscriptstyle 0}$ independently from $\CD_0$ and $\{\alpha_{i,j}\}$ uniformly from the symmetrical group.\\[1mm]
         \STATE For all $(\ell, j) \in [k]\times[m-1]$, if $S_0(\theta) = S_j^{(\ell)}(\theta)$, then let $U_{\ell,j} \doteq  \sign(\pi(j)- \pi(0)) \cdot \infty$;
         else
         \begin{align}
             U_{\ell,j}\,=\, \frac{\widehat{\mu}(\CD_0) - \frac{1}{|B_\ell|}\sum_{i \in B_\ell}\alpha_{i,j} X_i}{1 - \frac{1}{|B_\ell|}\sum_{i \in B_\ell} \alpha_{i,j}},
         \end{align}
         where $\frac{\pm c}{0} = \pm \infty$ for all $c > 0$.
         \STATE
         For all $(\ell, j) \in [k]\times[m-1]$, let\,
             $U_j \doteq\, \med_{\ell \in [k]} U_{\ell, j}$.
         \STATE 
         Return\, $U\,\doteq\, U_{(m-r)}$,
         where $U_{(1)}, \dots, U_{(m-1)}$ are ordered w.r.t.\ $\prec_\pi$.
    \end{algorithmic}
    \end{algorithm}

    \begin{lemma}\label{lemma:finite-rep}
        Assume A\ref{ass:iid} and A\ref{ass:symmetry}. For any $p = \nicefrac{r}{m}$ with  $1 \leq r < m$, we have 
        \begin{equation}
            U = U_{(m-r)}\qquad \text{with}\qquad
            U_{j} \doteq \med_{\ell \in [k]}\frac{\widehat{\mu}(\CD_0) - \frac{1}{|B_\ell|}\sum_{i\in B_\ell} \alpha_{i,1} X_i}{1 - \frac{1}{|B_\ell|}\sum_{i \in B_\ell} \alpha_{i,1}},
        \end{equation}
        for $j = 1, \dots, m-1$, where $U_{(1)} \prec_\pi \dots \prec_\pi U_{(m-1)}$ is ordered, and for notational simplicity we use $\frac{\pm c}{0} = \pm \infty$ for all $c > 0$, and\, $0/0 = \sign( \pi(1) - \pi(0)) \cdot\infty$.
    \end{lemma}
    
    In the theorem that follows we present a PAC bound on $U-\mu$ under A\ref{ass:iid}-A\ref{ass:finite_mom}. The proof can be found in the supplementary materials.

    \begin{theorem}\label{thm:ucb}
        Assume A\ref{ass:iid}, A\ref{ass:symmetry} and A\ref{ass:finite_mom}. For $1 \leq r \leq m$ user-chosen integers and for 
        \begin{equation}
        \begin{aligned}
        \Theta_n \doteq \{\, \theta : R(\theta)  \leq m-r\, \},
        \end{aligned}
        \end{equation}
        for every $n \in \NN$, integer $k \leq n$ and $\tilde{n} = \lfloor n/k \rfloor$ we have
        \begin{equation}
        \begin{aligned}
        \PP \bigg( \,U-\mu > 4\Big(\frac{(12M)^{1/a} }{\tilde{n}}\Big)^{\frac{a}{1+a}} \,\bigg)      \leq (m-r)\big(2k\exp(-\nicefrac{\tilde{n}}{8}) + 2\exp(-\nicefrac{k}{8})\big).
        \end{aligned}
        \end{equation}
    \end{theorem}

    \section{Resampled Median-of-Means UCB Algorithm}
    We present an important application of our scheme by deriving a new UCB algorithm for the problem of stochastic MAB. The upper confidence algorithms and their near-optimality were first published in \citep{auer2002finite} under subgaussian assumptions on the distributions. We elaborate on the work of \citep{bubeck2013bandits}, where heavy-tailed distribution were considered and near-optimal bounds were given. Our methods are proved to have the same regret bound up to a logarithmic factor, however, we do not need the knowledge of the moments nor its order to apply our algorithm under symmetricity.

    Let us consider the stochastic multiarmed bandit model with $K$ arms. For every arm a symmetrical distribution generates the rewards if one chooses that arm. 
    The reward distributions, $\nu_1, \dots, \nu_K$,  are unknown. Their finite expected values are denoted by $\mu_1, \dots, \mu_K$ as above. We pose only mild assumptions on the environment:
    \begin{assumption}\label{ass:bandit-1}
        For every $i \in [K]$ distribution $\nu_i$ is symmetric about $\mu_i$.
    \end{assumption}
    \begin{assumption}\label{ass:bandit-2}
        For every $i \in [K]$ there exists $a_i \in (0,1]$  such that $\EE_{\scriptscriptstyle{X_i \sim \nu_i}} [|X_i - \mu_i|^{1+a_i}]= M_i < \infty$.
    \end{assumption}
    We consider heavy-tailed reward distributions with different moment parameters which are not known for the agent.
    At every round $t$ from $1$ to a finite horizon $n$ we must choose an arm $I_t$ and then obtain a reward from $\nu_{I_t}$. Our main goal is to minimize the {\em regret} \eqref{def:regret}.

    Let us present the {\em robust resampling median-of-means based} upper confidence bound (RMM-UCB) policy. We first pull each arm once, then always choose the one with highest upper confidence bound computed with the one-sided RMM method presented above.
    More precisely, observe that for Algorithm \ref{alg:conf_interval_alg} one needs a sample $\CD$, a confidence level $p$ and a MoM parameter $k$ to construct the UCB. Let $A_i(t) = \{s: s\leq t, I_s =i\}$, $T_i(t) = \sum_{s=1}^t \BI(I_s = i)$ and $\CD_i(t-1) = \{X_i\}_{i\in A_i(t-1)}$. Let us use the simplified notation
    \begin{equation}\label{eq:ucb}
    \begin{aligned}
        U_i(T_i(t-1), p, k) \doteq
           U( \{X_t\}_{t \in A_i(t-1)}, p, k) \,\,\, \mbox{for all $i \in [K]$, $t \in \NN$ and $k \leq T_i(t-1)$.}
    \end{aligned}
    \end{equation}
    Then the RMM-UCB algorithm is defined by Algorithm \ref{alg:RSPSUCB}. We can observe that the method guarantees that $k_t^{(i)} \leq T_i(t-1)$ for all $t \in \NN$ and $i \in [K]$, therefore the UCBs are always well-defined. We emphasize that the algorithm is anytime and parameter-free, i.e., we do not need to know the horizon and the moment parameters a priori to apply the method. In Theorem \ref{thm:bandit-opt-regret} we present a near-optimal regret bound for the RMM-UCB algorithm. The proof can be found in the supplementary materials.
    \begin{theorem}\label{thm:bandit-opt-regret}
        Assume A\ref{ass:bandit-1} and A\ref{ass:bandit-2} and let  $c_i = 4^{\frac{1+a_i}{a_i}} \cdot 12^{1/(1+a_i)}$ for $i \in [K]$. Then, for the regret of the RMM-UCB policy we have 
        \begin{align}\label{eq:regret-bound}
            &R_n \leq \sum_{i:\Delta_i > 0} \bigg(\!\max\bigg\{c_i \bigg(\frac{M_i}{\Delta_i^{{1+a_i}}}\bigg)^{\!1/a_i}\!\!\!\!\!\!\!\!, \;\;\; 17^2\bigg\}\log^2(n) + C\bigg) \cdot \Delta_i.
        \end{align}
        Additionally, if $a = a_i$ and $M=M_i$ for $i \in [K]$, then for $n$ large enough we have 
        \begin{equation}\label{eq:regret-bound2}
            R_n \leq n^{\frac{1}{1+a}} (K 2c\log^3(n))^{\frac{a}{1+ a}}M^{\frac{1}{1+a}}.
        \end{equation}
    \end{theorem}
    
    \begin{algorithm}[t]
    \caption{RMM-UCB policy}
    \textbf{Inputs:} number of arms $K$\vspace{1mm}
    \begin{algorithmic}[1]\label{alg:RSPSUCB}
         \hrule
         \STATE Pull each arm once.
         \STATE At round $t$ let $p_t = \nicefrac{1}{\ceil{1 + t \log^2(t)}}$ and $k_{t}^{(i)} \doteq \lfloor 17 \log(t)\wedge\sqrt{T_i(t-1)}\rfloor$.
         \STATE At round $t \in \NN$ for every $i \in [K]$ compute $U_i(T_i(t-1), p_t, k_t^{(i)})$ as in \eqref{eq:ucb}
         \STATE Choose arm $I_t \doteq \argmax_{i \in [K]} U_i(T_i(t-1), p_t, k_t^{(i)})$.
    \end{algorithmic}
    \end{algorithm}

    \section{Numerical Experiments}\label{sec:experiments}
       In this section we numerically compare the performance of the proposed RMM-UCB algorithm with baseline concentration inequality based approaches, such as Vanilla UCB \citep{lattimore2020bandit}, Median-of-Means UCB \citep{bubeck2013bandits}, and Truncated Mean UCB \citep{bubeck2013bandits}. Furthermore, RMM-UCB is also compared with state-of-the-art data-driven approaches, such as Perturbed-History Exploration (PHE) \citep{kveton2020a}, Minimax Optimal Robust Adaptively Perturbed Exploration (MR-APE) \citep{lee2022minimax} and Maximum Average Randomly Sampled (MARS) \citep{khorasani2023maximum}.
       
       We consider a stochactic bandit setting with $K = 2$ arms, where $\mu^* = \mu_1 = 1$, $\mu_2 = \mu_1-\Delta$ and $\Delta \in (0,1]$ determines the suboptimality gap. As we study the MAB problem in case of heavy-tailed distributions, and we assume symmetric reward distributions, the rewards of both arms are sampled from a symmetrized Pareto distribution, more specifically $\forall i \in [K]: \nu_i \sim S(X-1)$, where $S$ is sampled from the Rademacher distribution and $X$ from a Pareto distribution with scale parameter $1$ and shape parameter $\alpha_p = 1.05 + \varepsilon_p$. 
       
       In all algorithms we set the round dependent confidence parameter as $\delta_t = 1 + t \log^2(t)$ ($p_t = 1/\delta_t$) and in case of the median-of-means estimators, the rewards of arm $i$ are divided into $k_{t} = \lfloor 17 \log(t)\wedge\sqrt{T_i(t-1)}\rfloor$ partitions as in Algorithm \ref{alg:RSPSUCB}. For the moment upper bound parameter $M$ of the Truncated Mean and Median-of-Means UCB algorithms, we used the best possible one, i.e., $M = \BE(|\nu_i|^{1+\varepsilon})$. The PHE hyperparameter was set to $a=5.1$, while we chose uniform perturbations and hyperparameters $c = 0.5$, $\epsilon = 0$ for the MR-APE.

       The average cumulative regret of each algorithm from $100$ independently generated trajectories in case of $\varepsilon_p = 0.1$ and suboptimality gaps $\Delta = 0.1$, $\Delta = 0.5$ are shown in Figure \ref{fig:avg_regret_main}. 
       It can be observed that for a small suboptimality gap $\Delta = 0.1$, the RMM-UCB outperforms all the other algorithms, which have similar average regrets. In case $\Delta = 0.5$, the RMM-UCB, MARS and MR-APE algorithms perform significantly better than the other solutions, however RMM-UCB still demonstrates the best performance.
       This results indicates that the proposed data-driven RMM-UCB can achieve lower cumulative regrets than existing solutions in case of difficult bandit problems, i.e. small suboptimality gaps and heavy-tailed reward distributions, while also having the advantage that it does not require any knowledge of moment or distribution parameters. Simulation results for different Pareto parameters can be found in the supplement. As we investigate bandit problems with heavy-tailed distributions, showing the standard deviations of the cumulative regrets in the same figure would reduce comprehensibility, therefore they are illustrated separately in the supplements.  

        \begin{figure*}[!b]
        \centering
        \subfloat[$\Delta = 0.1$]{\includegraphics[width=0.48\columnwidth]{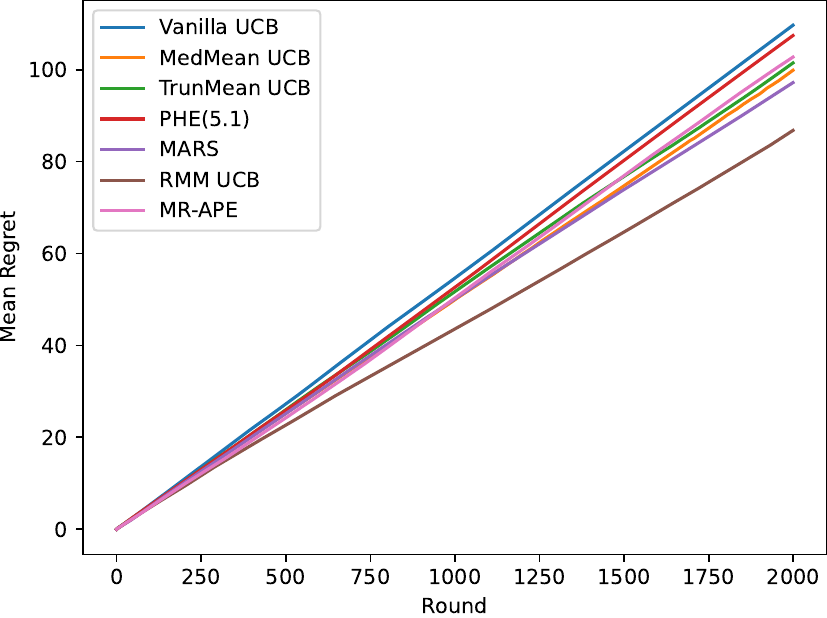}%
        \label{fig:reg_mean_0.1gap}}
        \hfil
        \hspace{3mm}
        \subfloat[$\Delta = 0.5$]{\includegraphics[width=0.48\columnwidth]{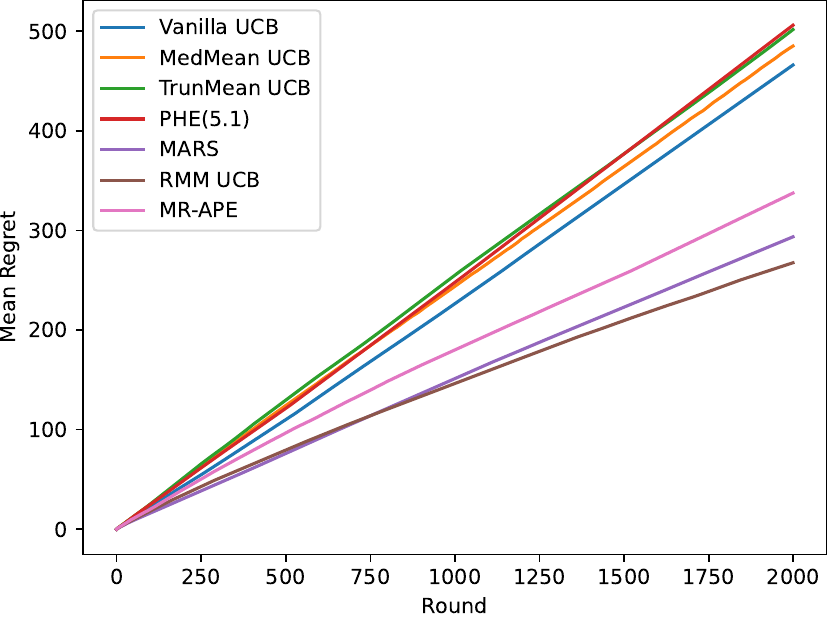}
        \label{fig:reg_mean_0.5gap}}
        \medskip
        \caption{Comparison of (average) cumulative regrets for Pareto bandits with $\varepsilon_p = 0.1$.}
        \label{fig:avg_regret_main}
        \end{figure*}
       
        \section{Conclusion}
        Stochastic multi-armed bandits (MABs) are fundamental and extensively studied online learning problems. UCB style algorithms serve as the standard ways to solve them. They can achieve near-optimal regret for many variants of MAB problems. The core and most important part of UCB methods is the construction of upper confidence bounds. The standard way to get such confidence bounds is to rely on concentration inequalities based on some moment assumptions. However, in many cases we do not have any a priori information on the reward distributions, we may not know any of their moment parameters either. This motivates completely data-driven methods which are free from such hyper-parameters that need to be tuned. In this paper, we built on the recently developed resampled median-of-means (RMM) estimator \citep{tamas2024data} and proposed an anytime, parameter-free UCB algorithm which is efficient even for the case of heavy-tailed rewards. We introduced a one-sided version of the RMM method to build UCBs with exact confidence levels, assuming only that the rewards are distributed symmetrically about their means. The data-driven MARS algorithm of \citet{khorasani2023maximum} can be seen as a special case of our construction, i.e., it corresponds to the case of taking the (resampled) median of only one mean ($k=1$). We proved a regret bound for the proposed RMM-UCB algorithm which is optimal up to a logarithmic factor, even under mild, heavy-tailed moment conditions. Unlike most previous constructions, RMM-UCB achieves this near-optimal regret without any a priori information on the moments, particularly, without the knowledge of parameters $a$ and $M$, cf.\ 
        Assumption \ref{ass:finite_mom}. RMM-UCB was also compared numerically with several baseline, concentration inequality based UCB methods, as well as with recent, state-of-the-art data-driven UCB algorithms on heavy-tailed bandit problems. The presented experiments are indicative of the phenomenon that RMM-UCB can outperform most previous algorithms on difficult MAB problems, i.e., when the suboptimality gap is small and the reward distributions are heavy-tailed.
                
    \bibliography{RMM_UCB_Arxiv}

    \newpage
    \begin{appendices}
    {
        \centering
        \section*{\LARGE \sc Supplementary Material}
        \vspace{3mm}
    }
                
    \section{Proof of Theorem \ref{thm:exact-coverage}}
 
    We prove that the one-sided RMM hypothesis test admits an exact confidence level. This result is based on a resampling rank test presented in \citep{csaji2014sign}. The key object of the hypothesis test is the so-called ranking function.
    \begin{definition}[ranking function]\label{def:ranking-function}
Let 
$\BA$ be a measurable space,
a (measurable) function 
$\psi : \BA^m \to [\,m\,]$ is called a 
{ranking function} if for all $(a_1, \dots, a_m) \in \BA^m$, it satisfies\vspace{-1mm}
\begin{enumerate}
\item[P1] For all permutation $\mu$ on $\{2,\dots, m\}$, 
$\psi\big(\,a_1, a_{2}, \dots, a_{m}\,\big)\; = \;
\psi\big(\,a_1, a_{\mu(2)}, \dots, a_{\mu(m)}\,\big),$
that is $\psi$ is invariant w.r.t. reordering the last $m-1$ terms of its arguments.
\item[P2] For all $i,j \in  [\,m\,]$,
if $a_i \neq a_j$, then we have
$\psi\big(\,a_i, \{a_{k}\}_{k\neq i}\,\big)\, \neq \;\psi\big(\,a_j, \{a_{k}\}_{k\neq j}\,\big)$,
where the simplified notation is justified by P1.
\end{enumerate}
\end{definition}

The value of a ranking function is called the {\em rank}. An important observation about the ranks of exchangeable random elements is \cite[Lemma 1]{csaji2019}:
\smallskip
\begin{lemma}\label{lemma:uniform}
{Let $\xi_1, \dots, \xi_m$ be a.s.\ pairwise different exchangeable random elements and let $\psi$ be a ranking function. Then $\psi(\xi_1, \dots, \xi_m)$ has a discrete uniform distribution on $[m]$}.
\end{lemma}
\smallskip
The original sample $\CD_0$ and the alternative samples $\CD_j(\theta)$ are random vectors in $\RR^n$. Observe that these datasets can be identical. This poses a technical challenge in ranking. The tie-breaking permutation $\pi$ is introduced to resolve this issue.
Let $\CD_j^\pi(\theta) \doteq (\CD_j(\theta), \pi(j))$. It is easy to prove that $\CD_0^\pi(\mu), \dots, \CD_{m-1}^\pi(\mu)$ are exchangeable, hence we obtain an exact hypothesis test if we reject $H_0$ if and only if $\psi(\CD_0^\pi(\theta), \dots, \CD_{m-1}^\pi(\theta)) > m-r$. We refer to \citep[Theorem 3.1]{tamas2024data}:

\smallskip
\begin{theorem}\label{thm:exact-cov-general-rank}
    For any ranking function $\psi$ if A\ref{ass:iid} and A\ref{ass:symmetry} hold we have
    \begin{equation}\label{eq:exact}
        \PP\big(\, \psi\big(\,\CD_0^\pi(\mu), \dots, \CD_{m-1}^\pi(\mu)\,\big) > m- r\, \big) = 1 - \frac{r}{m}. 
    \end{equation}
\end{theorem}

Theorem \ref{thm:exact-coverage} is a direct consequence of Theorem \ref{thm:exact-cov-general-rank}, because $R$ is a ranking function.

\section{Proof of Lemma \ref{lemma:finite-rep}}

    Because of Lemma \ref{lemma:finite-rep-m-2} we have
            \begin{equation}
                \{ \theta: S_0(\theta) \succ_\pi S_j(\theta)\} = (-\infty, U_j]_\pi
            \end{equation}
            for $j = 1, \dots, m-1$. Henceforth, it is easy to see that 
            \begin{equation}
                \bigg\{ \theta: 1 + \sum_{j=1}^{m-1} \BI( S_0(\theta) \prec_\pi S_j(\theta)) \leq m-r \bigg\} = (-\infty, U_{(m-r)}]_\pi,    
            \end{equation}
            therefore $U = U_{(m-r)}$. \qed
    \begin{lemma}\label{lemma:finite-rep-m-2}
        Assume A\ref{ass:iid} and A\ref{ass:symmetry}. For $m=2$ 
        \begin{equation}
            U = \med_{\ell \in [k]}\frac{\widehat{\mu}(\CD_0) - \frac{1}{|B_\ell|}\sum_{i\in B_\ell} \alpha_{i,1} X_i}{1 - \frac{1}{|B_\ell|}\sum_{i\in B_\ell} \alpha_{i,1}},
        \end{equation}
         where $\frac{\pm c}{0} = \pm \infty$ for all $c > 0$ and $0/0 = \sign( \pi(1) - \pi(0)) \cdot\infty$.
    \end{lemma}
    \begin{proof}
     The graph of reference function 
        \begin{equation}
            S_0(\theta)  = \widehat{\mu}(\CD_0) - \theta
        \end{equation}
        is a line with slope $-1$, whereas
        \begin{align}
            S_1(\theta) &= \widehat{\mu}(\CD_1(\theta)) - \theta \nonumber\\
            &= \med \bigg( \frac{1}{|B_1|}\sum_{i \in B_1} (\alpha_{i,1} ( X_i - \theta) + \theta),\nonumber\dots ,\frac{1}{|B_k|}\sum_{i\in B_k} (\alpha_{i,1} ( X_i - \theta) + \theta)\bigg) - \theta\nonumber\\
            &= \med \bigg(  \frac{1}{|B_1|}\sum_{i \in B_1} \alpha_{i,1} ( X_i - \theta), \dots ,\frac{1}{|B_k|}\sum_{i\in B_k} \alpha_{i,1} ( X_i - \theta) \bigg),
        \end{align}
        which is the median of linear functions with slopes between $-1$ and $+1$.
        Let
        \begin{equation}
            S_1^\ell(\theta) \doteq  \frac{1}{|B_\ell|}\sum_{i \in B_\ell} \alpha_{i,1} ( X_i - \theta) \text{ for }\ell = 1, \dots, k
        \end{equation}
        be sub-sample linear functions. It is easy to see that
        \begin{equation}
            \{ \theta: S_1^{\ell}(\theta) < S_0(\theta)\} = (-\infty, \nu_\ell),
        \end{equation}
        where if $S_0$ and $S_1^{\ell}$ are not parallel, then $\nu_\ell$ is the intersection of $S_0$ with $S_1^{\ell}$, in which case
        \begin{equation}\label{eq:form-U}
        \begin{aligned}
            \nu_\ell &= \frac{\widehat{\mu}(\CD_0) - \frac{1}{|B_\ell|}\sum_{i \in B_\ell} \alpha_{i,1} X_i}{1 - \frac{1}{|B_\ell|}\sum_{i \in B_\ell} \alpha_{i,1}}\\
            &= \mu + \frac{\widehat{\mu}(\mathcal{W}_0) -  \frac{1}{|B_\ell|}\sum_{i \in B_\ell} \alpha_{i,1} W_i}{\frac{1}{|B_\ell|}\sum_{i \in B_\ell} (1- \alpha_{i,1})}\\
            &= \mu + \frac{\widehat{\mu}(\mathcal{W}_0) - \frac{1}{|B_\ell|}\sum_{i \in B_\ell} \alpha_{i,1} W_i}{\frac{2}{|B_\ell|} Z_\ell}.
        \end{aligned}
        \end{equation}
        If $S_0$ and $S_1^{\ell}$ are parallel, then the extended ``intersection'' points are
        \begin{equation}
            \nu_\ell \doteq
	  \begin{cases}
		  +\infty  & \mbox{if } \forall \theta: S_1^{\ell}(\theta)\prec_\pi S_0(\theta)\\
		  -\infty & \mbox{if } \forall \theta:  S_0(\theta) \prec_\pi S_1^{\ell}(\theta) .
	  \end{cases}
        \end{equation}
        These values are equivalent to the formula of \eqref{eq:form-U} with $\pm c/ 0 = \pm \infty$ and slightly abuse of notation $0/0 = \sign( \pi(1) - \pi(0)) \cdot\infty$.
        By
        \begin{equation}
        \begin{aligned}
            \{\theta: S_0(\theta) \succ_\pi S_1(\theta)\}
            &= \bigcap_{I \subseteq [n], |I| =\lfloor k/2 \rfloor} \bigcap_{\ell \in I} \{ \theta: S_0(\theta) \succ_\pi S_1^{\ell}(\theta) \}\\
            &= \bigcap_{I \subseteq [n], |I| =\lfloor k/2 \rfloor+1} \bigcap_{\ell \in I} ( -\infty, \nu_\ell]_\pi = (-\infty, \nu]_\pi,
        \end{aligned}
        \end{equation}
        where $(-\infty, \nu]_\pi$ is closed from the right depending on $\pi$, i.e., $(-\infty, \nu]_\pi \doteq ( -\infty, \nu]$ if $\pi(0) > \pi(1)$ and closed otherwise.
        Henceforth,
        \begin{align}
            U&= \med_{1 \leq \ell \leq k} \nu_\ell \nonumber = \mu + \med_{1 \leq \ell \leq k} \bigg( \frac{\widehat{\mu}(\mathcal{W}_0)-  \frac{1}{|B_\ell|}\sum_{i\in B_\ell} \alpha_{i,1} W_i}{\frac{2}{\tilde{n}}Z_\ell}\bigg)\nonumber = \mu + V,
        \end{align}
        where $V$ is defined by
        \begin{equation}
            V \doteq \med_{1 \leq \ell \leq k} \bigg( \frac{\widehat{\mu}(\mathcal{W}_0)-  \frac{1}{|B_\ell|}\sum_{i \in B_\ell} \alpha_{i,1} W_i}{\frac{2}{|B_\ell|} Z_\ell}\bigg)
        \end{equation}
         is a valid $50\%$ UCB.
        \end{proof}

        \section{Proof of Theorem \ref{thm:ucb}}

        The proof is based on Theorem \ref{thm:diameter-a} and the union bound. It is essentially the same as the proof of \citep[Theorem 3]{szentpeteri2023sample}. \qed
        
        \begin{theorem}\label{thm:diameter-a}
        Assume A\ref{ass:iid}-A\ref{ass:finite_mom}. For $m=2$ and
        \begin{equation}
        \begin{aligned}
        U \doteq \sup\,\{\, \theta : R(\theta)  = 1 \},
        \end{aligned}
        \end{equation}
        for every $k \leq n$, $n\in \NN$ and $\tilde{n} = \lfloor n/k\rfloor$, we have
        \begin{equation}
        \begin{aligned}
        &\PP \bigg( \,U - \mu> 4(12M)^{1/(1+a)}\bigg(\frac{1}{\tilde{n}}\bigg)^{\frac{a}{1+a}} \,\bigg) \leq 2k\exp(-\nicefrac{\tilde{n}}{8}) + 2\exp(-\nicefrac{k}{8})
        \end{aligned}
        \end{equation}
    \end{theorem}

    \begin{proof}
        The proof consists of two parts as we bound 
        \begin{equation}\label{eq:diam-prob}
            \PP(U-\mu > \varepsilon )
        \end{equation}
        with two different terms. Let
        \begin{equation}
        \begin{aligned}
            Z_\ell \doteq \frac{1}{2}\sum_{i \in B_\ell} (1- \alpha_i) \quad \text{ for} \quad \ell = 1, \dots, k. 
        \end{aligned}
        \end{equation}
        We observe that $\{Z_\ell\}_{\ell=1}^k$ are independent binomial variables with parameters $|B_\ell|$, $\nicefrac{1}{2}$ and expected value $\nicefrac{|B_\ell|}{2}$ for $\ell \in [k]$. Let us denote the events that follow by
        \begin{equation}
        \begin{aligned}
            &A_\ell \doteq \big\{|Z_\ell - \EE Z_\ell| \leq \nicefrac{|B_\ell|}{4} \big\} \quad \text{for}\quad \ell = 1, \dots, k \quad \text{and} \quad A \doteq \bigcap_{\ell=1}^k A_\ell. 
        \end{aligned}
        \end{equation}
        By the law of total probability
        \begin{align}
            \PP ( U - \mu > \varepsilon ) &= \PP( U - \mu > \varepsilon \,|\, A ) \PP(A) + \PP(  U - \mu > \varepsilon \,|\, \bar{A} ) \PP(\bar{A})\\  &\leq \PP( \{ U-\mu > \varepsilon \} \cap A ) + \PP(\bar{A}).
        \end{align}
        The second term is bounded from above by the union bound and Hoeffding's inequality as
        \begin{equation}
        \begin{aligned}
            &\PP(\,\bar{A}\,) \leq \PP\bigg( \,\bigcup_{\ell=1}^k \bar{A}_\ell \bigg)
            \leq \sum_{\ell =1}^k\PP(  |Z_\ell - \EE Z_\ell| > \nicefrac{|B_\ell|}{4} ) \\
            &\leq \sum_{\ell = 1}^k 2 \exp\bigg( -\frac{|B_\ell|}{8}\bigg) \leq 2 k \exp\bigg( -\frac{\tilde{n}}{8}\bigg)
        \end{aligned}
        \end{equation}
        For the first term recall that
        \begin{align}\label{eq:or}
            &\PP( U-\mu > \varepsilon \,|\,A ) = \PP( V > \varepsilon \,|\,A)\nonumber.
        \end{align}
        We observe that $Z_\ell$ is independent of the nominator, because $\alpha_{i,1}W_i$ and $\alpha_{i,1}$ are independent for each $i \in [n]$. Let us consider the following events:
        \begin{align}
            &B = \{V > \varepsilon \}, \quad \tilde{B} = \bigg\{\med_{\ell \in [k]} \bigg( \frac{\widehat{\mu}(\mathcal{W}_0)-  \frac{1}{|B_\ell|}\sum_{i \in B_\ell} \alpha_{i,1} W_i}{\frac{1}{2}}\bigg) > \varepsilon\bigg\}.
        \end{align}  
        Our key observation is that
        \begin{equation}
            B \cap A \subseteq \tilde{B} \cap A,
        \end{equation}
        because if $V$ is positive, then decreasing $Z_\ell$ in the denominator for every $\ell=1, \dots, k$ down to $\nicefrac{\tilde{n}}{4}$, increases the median.
        Consequently
        \begin{align}
            &\PP(B \cap A) \leq \PP ( \tilde{B} \cap A) \nonumber \\
            &\leq \PP\bigg( \, \med_{\ell \in [k]} \bigg( \widehat{\mu}(\mathcal{W}_0)-  \frac{1}{|B_\ell|}\sum_{i \in B_\ell} \alpha_{i,1} W_i\bigg) > \nicefrac{\varepsilon}{2}\bigg) \nonumber \\
            &= \PP\bigg( \, \widehat{\mu}(\mathcal{W}_0) -\med_{\ell \in [k]} \bigg(\frac{1}{|B_\ell|}\sum_{i \in B_\ell} \alpha_{i,1} W_i\bigg) > \nicefrac{\varepsilon}{2}\bigg) \nonumber \\
            & \leq \PP \bigg( \{\widehat{\mu}(\mathcal{W}_0) > \nicefrac{\varepsilon}{4}\} \cup \bigg\{-\med_{\ell \in [k]} \frac{1}{|B_\ell|}\sum_{i \in B_\ell} \alpha_{i,1} W_i > \nicefrac{\varepsilon}{4}\bigg\} \bigg)\nonumber \\
            &\leq \PP \big( \widehat{\mu}(\mathcal{W}_0) > \nicefrac{\varepsilon}{4}\big) + \PP \bigg(\med_{\ell \in [k]} \frac{1}{|B_\ell|}\sum_{i \in B_\ell} -\alpha_{i,1} W_i > \nicefrac{\varepsilon}{4}\bigg)\nonumber\\
            & = 2 \cdot \PP \big( \widehat{\mu}(\mathcal{W}_0) > \nicefrac{\varepsilon}{4}\big).
        \end{align}
        hence for $\varepsilon = 4\Big(\frac{(12M)^{1/a} }{\tilde{n}}\Big)^{\frac{a}{1+a}}$ by Theorem \ref{thm:bubeck-nonequalgroups} we have
        \begin{equation}
        \begin{aligned}
            &\PP\bigg(\,U - \mu > 4\Big(\frac{(12M)^{1/a} }{\tilde{n}}\Big)^{\frac{a}{1+a}}\, \bigg) \leq 2k\exp\bigg( -\frac{\tilde{n}}{8}\bigg) + 2 \exp\bigg( -\frac{k}{8}\bigg)
        \end{aligned}
        \vspace*{-2mm}
        \end{equation}
        \vspace*{-2mm}
        \end{proof}

        \section{Proof of Theorem \ref{thm:bandit-opt-regret}}
        By Lemma \ref{lemma:expectation-bound} and Lemma \ref{lemma:reg-decomp} Equation \eqref{eq:regret-bound} follows. For Equation \eqref{eq:regret-bound2} one can follow the proof of \citep[Proposition 1]{bubeck2013bandits}, which relies mostly on Hölder's inequality, i.e., for $n$ large enough, i.e., if for all $i\in [K]$ 
            \begin{equation}
            \begin{aligned}
                c \frac{M^{1/a}}{\Delta_i^{(1+a)/a}}\log(n) \geq  17^2 \quad \text{and} \quad c \frac{M^{1/a}}{\Delta_i^{(1+a)/a}}\log^3(n) \geq C, 
            \end{aligned}
            \end{equation}
            then we have
            \begin{equation}
            \begin{aligned}
                &R_n \leq \sum_{i: \Delta_i > 0} \Delta_i (\EE[ T_i(n)] ) ^{\frac{1}{1+a}} (\EE[T_i(n)] )^{\frac{a}{1+a}} \\
                &\leq \sum_{i: \Delta_i > 0} \Delta_i (\EE[ T_i(n)] ) ^{\frac{1}{1+a}}\bigg(\max\bigg(c \frac{M^{1/a}}{\Delta_i^{(1+a)/a}},17^2\bigg) \log^2(n) + C\bigg) ^{\frac{a}{1+a}}\\
                &\leq \sum_{i: \Delta_i > 0} \Delta_i (\EE[ T_i(n)] ) ^{\frac{1}{1+a}} \bigg(2c \frac{M^{1/a}}{\Delta_i^{(1+a)/a}} \log^3(n) \bigg) ^{\frac{a}{1+a}}\\
                & \leq \bigg(\sum_{i: \Delta_i > 0} \EE[ T_i(n)] \bigg)^{\frac{1}{1+a}} K^{\frac{a}{1 + a}} (2c\log^3(n))^{\frac{a}{1+ a}}M^{\frac{1}{1+ a}}\\
                &\leq n^{\frac{1}{1+a}} (K 2c\log^3(n))^{\frac{a}{1+ a}}M^{\frac{1}{1+ a}}.
                \end{aligned}
            \end{equation}
            In the last step we used that $\sum_{i=1}^K T_i(n) = n$ holds for all $n \in \NN$.\qed
        
        \begin{lemma}\label{lemma:expectation-bound}
        If we apply the RMM-UCB policy, then for $i \in [K]$, $i \neq 1$
        for $c_i \doteq 4^{\frac{1+a_i}{a_i}} \cdot12^{1/a_i}$
        \begin{equation}
          \EE [T_i(n)] \leq \max\bigg(c_i\bigg(\frac{M_i}{\Delta_i^{1+a_i}}\bigg)^{1/a_i}, 17^2\bigg) \log^2(n) + C. 
        \end{equation}
        \end{lemma}
        \begin{proof}
        Our proof strongly relies on \citep{bubeck2013bandits} and \citep{khorasani2023maximum}.        
        First observe that if $I_t = i$ for $i \neq 1$, then at least one of the following inequalities holds
        \begin{align}
            &U_1(T_1(t-1), p_t, k_t^{(1)}) \leq \mu_1,\label{eq:a}\\
            &U_i(T_i(t-1),  p_t, k_t^{(i)}) > \mu_i + M_i^{\frac{1}{1+a_i}} \bigg( \frac{c_i  k_t^{(i)}}{T_i(t-1)}\bigg)^{\frac{a_i}{1+a_i}}, \label{eq:b}\\
            &T_i(t-1) <  u_i
        \end{align}
        with $c_i \doteq 4^{\frac{1+a_i}{a_i}} \cdot12^{1/a_i}$ and
        \begin{equation*}
            u_i \doteq \bigg\lceil \max\bigg(c_i \bigg(\frac{M_i}{\Delta_i^{(1+a_i)}}\bigg)^{1/a_i}, 17^2\bigg)\log^2(n)\bigg\rceil.
        \end{equation*}
        If all of them were false, then
        \begin{equation}
        \begin{aligned}
            &U_1(T_1(t-1), p_t, k_t^{(1)})> \mu_1 = \mu_i + \Delta_i       \geq \mu_i + M_i^{1/(1+a_i)} \bigg( \frac{ c_i \log^2(n)}{T_i(t-1)} \bigg)^{a_i/(1+a_i)}\\
            &\geq \mu_i + M_i^{1/(1+a_i)} \bigg( \frac{c k_t^{(i)}}{T_i(t-1)} \bigg)^{a_i/(1+a_i)} \geq U_i(T_i(t-1),  p_t, k_t^{(i)}),
        \end{aligned}
        \end{equation}
        which contradicts $I_t =i$. We prove that \eqref{eq:a} or \eqref{eq:b} occur with small probability. Recall that $U_i = \mu_1 + V_i$. Let us denote the bad events above as
        \begin{equation*}
        \begin{aligned}
            &B_t^{(1)} \doteq \{V_1(T_1(t-1),  p_t, k_t^{(1)}) \leq 0\},\\
            &B_t^{(2)} \doteq \bigg\{V_i(T_i(t-1),  p_t, k_t^{(i)}) > M_i^{1/(1+a_i)} \bigg( \frac{c_i k_t^{(i)}}{T_i(t-1)}\bigg)^{a_i/(1+a_i)}\bigg\}.
        \end{aligned}
        \end{equation*}
        and the good events as
        \begin{equation*}
        \begin{aligned}
            &G_t \doteq \{\,T_i(t-1) < u_i\, \}
        \end{aligned}
        \end{equation*}
        for $t = 1, \dots, n$. By the first observation we have
        \begin{equation}
        \begin{aligned}
            &\EE[ T_i(n)] =  \EE\bigg[\sum_{t=1}^n \BI(I_t =i) \bigg] = \EE\bigg[\sum_{t=1}^n \BI(I_t =i)\BI(G_t)\bigg] + \EE\bigg[\sum_{t=u_i}^n \BI(I_t =i)\BI(\bar{G}_t)\bigg]\\ 
            &\leq u_i + \sum_{t=u_i}^n \PP((B_t^{(1)} \cup B_t^{(2)}) \cap \bar{G}_t) \leq u_i +  \sum_{t=u_i}^n \PP(B_t^{(1)}) + \sum_{t=u_i}^n \PP(B_t^{(2)} \cap \bar{G}_t) . 
        \end{aligned}
        \end{equation}

        For each $i \in [K]$
        by  Theorem \ref{thm:exact-coverage}
        \begin{equation}
            \PP(B_t^{(1)}) = p_t = \frac{1}{\ceil{1 + t \log^2(t)}}.
        \end{equation}
        Additionally, we have $\sqrt{u_i} \geq 17 \log(t)$ for $t \leq n$, henceforth by Theorem \ref{thm:ucb} we have
        \begin{equation}
        \begin{aligned}
            &\PP(B_t^{(2)} \cap \bar{G}_t) \\
            &\leq \PP\bigg(V_i(T_i(t-1),  p_t, k_t^{(i)})  >  M_i^{\frac{1}{1+a_i}} \bigg( \frac{c_i k_t^{(i)}}{T_i(t-1)}\bigg)^{\frac{a_i}{1+a_i}}, T_i(t-1) \geq u_i \bigg)\\
            &= \sum_{s=u_i}^t \PP\bigg(V_i(s,  p_t, k_t^{(i)})  >  M_i^{\frac{1}{1+a_i}} \bigg( \frac{c_i k_t^{(i)}}{s}\bigg)^{\frac{a_i}{1+a_i}}\,|\,T_i(t-1)=s\bigg) \cdot\PP(T_i(t-1) = s)\\
            &\leq \sum_{s=u_i}^t 2\ceil{t\log^2(t)}\bigg(k_t^{(i)}\exp\bigg(-\frac{s}{k_t^{(i)} 8}\bigg) + \exp\bigg( -\frac{k_t^{(i)}}{8}\bigg)\bigg) \cdot\PP(T_i(t-1) = s)\\
            &\leq \sum_{s=u_i}^t 2\ceil{t\log^2(t)}\cdot\\
            & \bigg(\sqrt{s}\wedge 17\log(t)\exp\bigg(-\frac{s}{\lfloor\sqrt{s}\wedge17\log(t)\rfloor \cdot 8}\bigg) + \exp\bigg( -\frac{\sqrt{s}\wedge 17\log(t) }{8}\bigg)\bigg) \cdot\PP(T_i(t-1) = s)\\
            &\leq \sum_{s=u_i}^t 2\ceil{t\log^2(t)}\bigg(17\log(t) \exp\bigg(-\frac{\sqrt{s}}{8}\bigg) + \exp\bigg( -\frac{\sqrt{s} \wedge 17\log(t)}{8}\bigg)\bigg) \cdot\PP(T_i(t-1) = s)\\
            &\leq 2\ceil{t\log^2(t)}\bigg(17\log(t) \exp\bigg(-\frac{\sqrt{u_i}}{8}\bigg) + \exp\bigg( -\frac{\sqrt{u_i} \wedge 17\log(t)}{8}\bigg)\bigg)\\
            &\leq 2\ceil{t\log^2(t)}\bigg(17 \log(t)\exp\bigg(-\frac{17\log(t)}{8}\bigg) + \exp\bigg( -\frac{17\log(t)}{8}\bigg)\bigg)\\
            &\leq 2\ceil{t\log^2(t)}\bigg(17 \log(t)\frac{1}{t^{2 + \gamma}} + \frac{1}{t^{2 + \gamma}}\bigg),
        \end{aligned}
        \end{equation}
        where $\gamma = 1/8$. We can observe that the right hand side is summable in $t$.
        Hence, the key quantity can be bounded from above as
        \begin{equation}
        \begin{aligned}
            &\EE[T_i(n)]
            \leq u_i +  \sum_{t=u_i}^n \bigg( \frac{1}{1 + \ceil{t \log^2(t)}} + \frac{\tilde{c}\log^3(t)}{t^{1+\gamma}}\bigg) \leq u_i + C.
        \end{aligned}
        \vspace*{-2mm}
        \end{equation}
        \vspace*{-2mm}
        \end{proof}

        \section{Auxiliary Results}

        We refer to fundamental results from \citep{bubeck2013bandits}:
        \begin{lemma}\label{lemma:bub}
             Assume A\ref{ass:bandit-1} and A\ref{ass:bandit-2}. Let $\bar{\mu}$ be the empirical mean. Then for any $\delta \in (0,1)$ 
            \begin{equation}
                \PP\bigg( \bar{\mu} \leq \mu + \bigg( \frac{3M}{\delta n^a}\bigg) ^{\frac{1}{1+a}} \bigg) \geq 1 -\delta.
            \end{equation}
        \end{lemma}

        Theorem \ref{theorem:bubeck} from \cite{bubeck2013bandits} is our main tool to prove the concentration inequality for MoM estimates.
        \begin{theorem}\label{theorem:bubeck}
            Let $\delta \in (0,1)$ and $a \in (0, 1]$. Assume A\ref{ass:iid}, A\ref{ass:finite_mom} and $n = k \tilde{n}$. Let $k = \lfloor \min(8 \log(e^{1/8}/\delta), n/2)\rfloor$, then for the MoM estimator $\widehat{\mu}$ we have
            \begin{equation}
                \PP\bigg( \widehat{\mu} \leq \mu + (12M)^{\frac{1}{1+a}}\bigg( \frac{16 \log(e^{1/8}/\delta)}{n}\bigg) ^{\frac{a}{1+a}} \bigg) \geq 1 -\delta.
            \end{equation}
        \end{theorem}
        We use a reparameterized version of Theorem \ref{theorem:bubeck}. The proof for almost equal group sizes is included for the sake of completeness.
        \begin{theorem}\label{thm:bubeck-nonequalgroups}
            Assume A\ref{ass:iid} and A\ref{ass:finite_mom}. Let $k \leq n$ and $\tilde{n} = \lfloor n/k \rfloor$, then for the MoM estimator $\widehat{\mu}$,
            \begin{equation}
                \PP\bigg( \widehat{\mu} > \mu + (12M)^{\frac{1}{1+a}}\bigg( \frac{1}{\tilde{n}}\bigg) ^{\frac{a}{1+a}} \bigg) \leq \exp\bigg( - \frac{k}{8}\bigg).
            \end{equation}
        \end{theorem}
        \begin{proof}
            Let $\eta > 0$ and $Y_\ell \doteq \BI( \bar{\mu}_\ell > \mu + \eta)$ for $\bar{\mu}_\ell \doteq \frac{1}{|B_\ell|} \sum_{i \in B_\ell} X_i$ and $\ell = 1, \dots, k$. By Lemma \ref{lemma:bub} $Y_\ell$ is a Bernoulli variable with\vspace{-2mm}
            \begin{equation}
                p_\ell \leq \frac{3M}{|B_\ell|^a \eta^{1+a}}.
            \end{equation}
            For\vspace{-2mm}
            \begin{equation}
                \eta = (12M)^{1/(1+a)} \bigg(\frac{1}{\tilde{n}}\bigg)^{\frac{a}{1+a}}
            \end{equation}
            we have $p_\ell \leq 1/4$ for all $\ell \in [k]$. Finally, by Hoeffding's inequality for a binomial variable $Z$ with parameters $k$ and $1/4$ we have
            \begin{equation}
                \PP( \widehat{\mu}  > \mu + \eta ) \leq \PP( Z \geq k/2 ) \leq \exp\bigg( - \frac{k}{8}\bigg).
                \vspace{-4mm}
            \end{equation}
        \end{proof}
        
        The regret decomposition lemma \cite[Lemma 4.5]{lattimore2020bandit} is one of the main tools to prove Theorem \ref{thm:bandit-opt-regret}.
        \begin{lemma}\label{lemma:reg-decomp}
        For any policy and stochastic bandit environment $\{\nu_1, \dots, \nu_K\}$  and horizon $n \in \NN$ the regret $R_n$ of the policy in the environment satisfies\vspace{-1mm}
        \begin{equation}
            R_n = \sum_{i=1}^K \Delta_i \EE[T_i(n)].
        \end{equation}
    \end{lemma}

    \section{Numerical Experiments}
    \subsection{Experiment: Average cumulative regret for $\varepsilon_p = 0.5$}
    This experiment presents the average cumulative regret of each algorithm given $\varepsilon_p = 0.5$, every other simulation setting is the same as in Section \ref{sec:experiments}. The results are illustrated in Figure \ref{fig:avg_regret_supplement}. It shows, that for $\Delta = 0.1$, the proposed RMM-UCB still outperforms all the other algorithms, however for the larger suboptimality gap $\Delta = 0.5$, the cumulative regret of MR-APE is smaller than RMM-UCB. Despite this, we conclude that the  RMM-UCB can achieve lower cumulative regrets compared to existing solutions in case of difficult bandit problems, i.e., when the suboptimality gap is smaller and the arms distributions are heavy-tailed.
    \begin{figure*}[!t]
    \centering
    \subfloat[$\Delta = 0.1$]{\includegraphics[width=0.48\columnwidth]{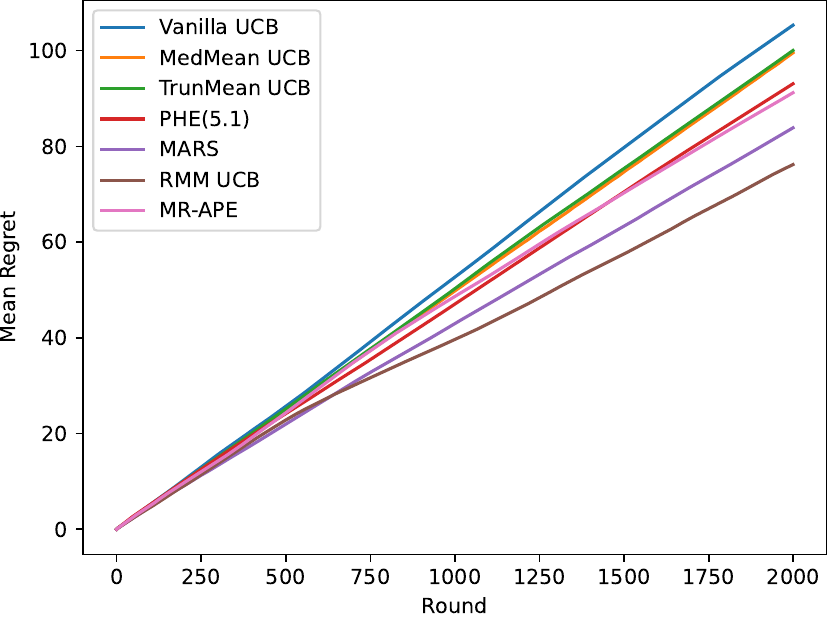}%
    \label{fig:reg_mean_0.1gap_pa1.55}}
    \hspace{3mm}
    \hfil
    \subfloat[$\Delta = 0.5$]{\includegraphics[width=0.48\columnwidth]{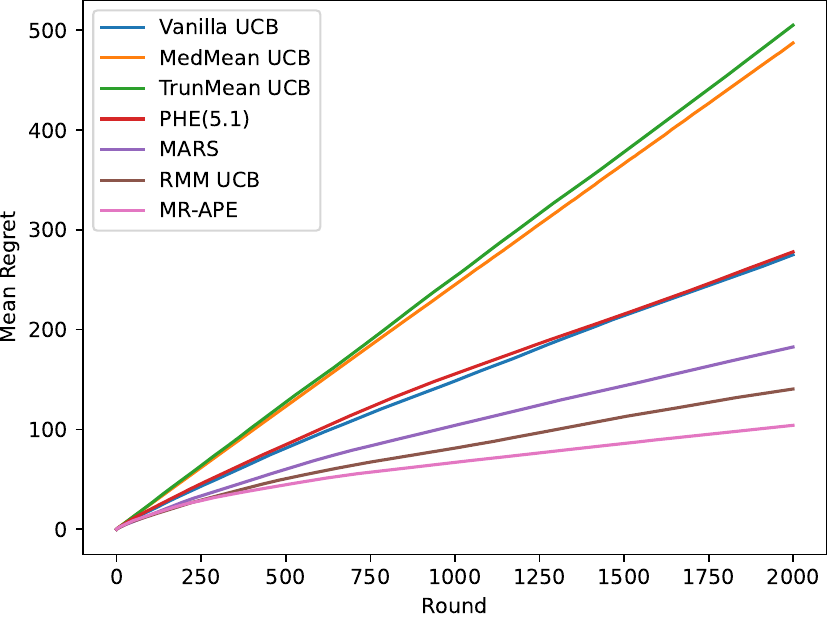}
    \label{fig:reg_mean_0.5gap_pa1.55}}
    \caption{Comparison of average cumulative regrets for Pareto bandits, $\varepsilon_p = 0.5$.}
    \label{fig:avg_regret_supplement}
    \end{figure*}

    \begin{figure*}[!t]
    \centering
    \subfloat[$\Delta = 0.1$]{\includegraphics[width=0.48\columnwidth]{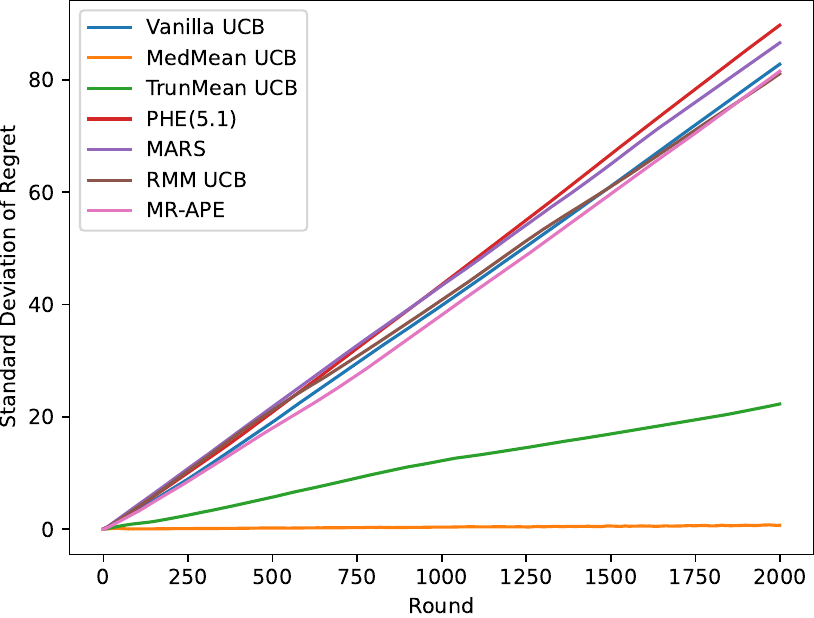}%
    \label{fig:reg_std_0.1gap_pa1.55}}
    \hspace{3mm}
    \hfil
    \subfloat[$\Delta = 0.5$]{\includegraphics[width=0.48\columnwidth]{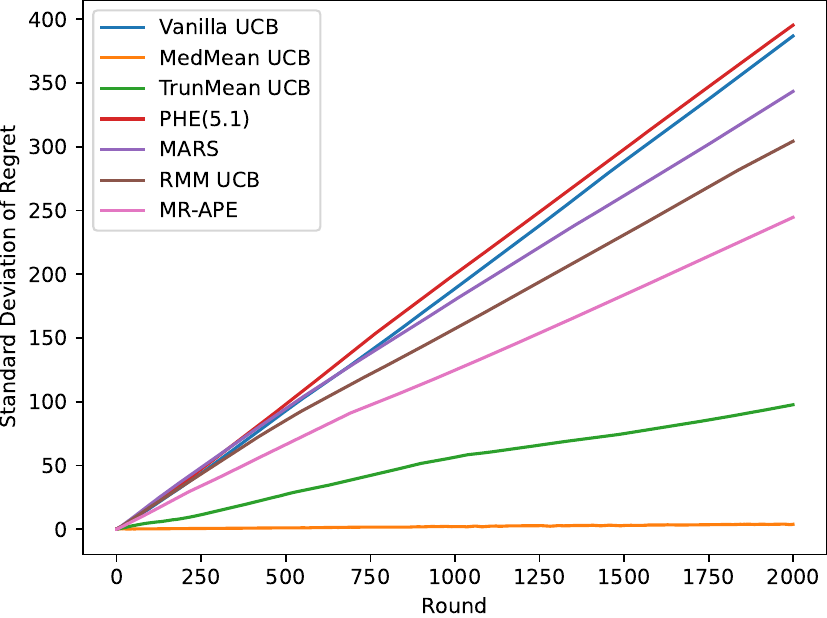}
    \label{fig:reg_std_0.5gap_pa1.55}}
    \caption{Comparison of std.\ deviations of cumulative regrets for Pareto bandits, $\varepsilon_p = 0.5$.}
    \label{fig:std_regret_2}
    \end{figure*}
    
    \begin{figure*}[!t]
    \centering
    \subfloat[$\Delta = 0.1$]{\includegraphics[width=0.48\columnwidth]{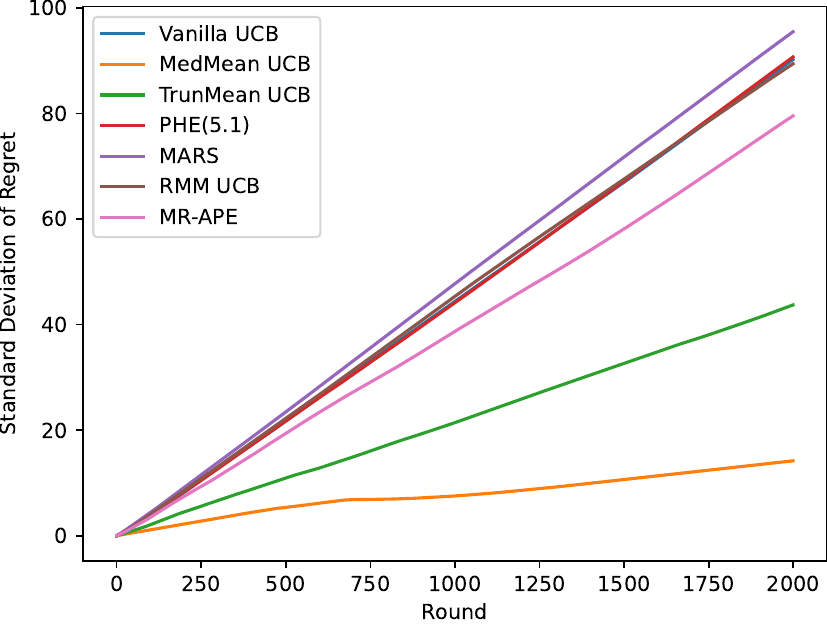}%
    \label{fig:reg_std_0.1gap_pa1.15}}
    \hspace{3mm}
    \hfil
    \subfloat[$\Delta = 0.5$]{\includegraphics[width=0.48\columnwidth]{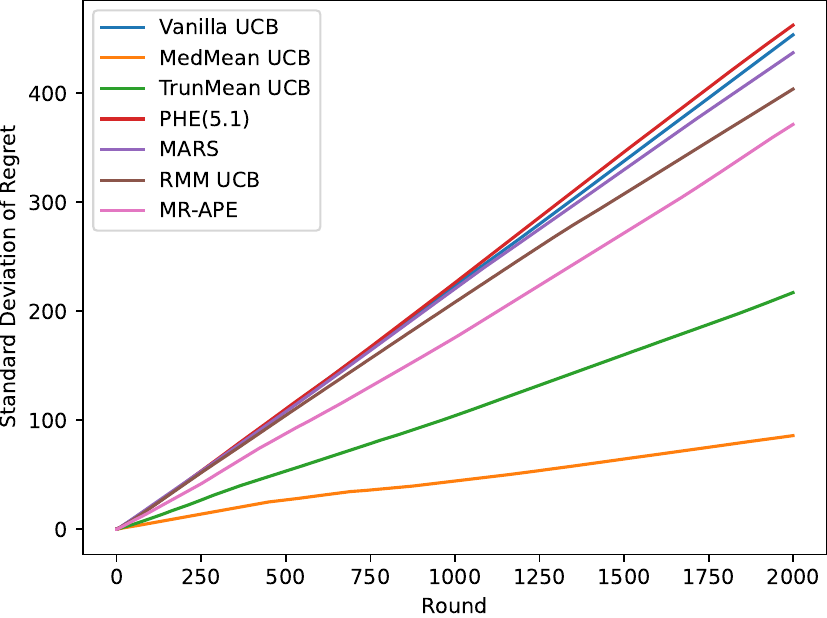}
    \label{fig:reg_std_0.5gap_pa1.15}}
    \caption{Comparison of std.\ deviations of cumulative regrets for Pareto bandits, $\varepsilon_p = 0.1$.}
    \label{fig:std_regret_1}
    \end{figure*}

    \subsection{Experiment: Standard deviation of cumulative regrets}

    In this experiment we investigate the standard deviations of the cumulative regrets for the same algorithms and simulation setting as in Section \ref{sec:experiments}. The results for different $\varepsilon_p$ and $\Delta$ parameters are shown in Figures \ref{fig:std_regret_2} and \ref{fig:std_regret_1}. It can be seen, that the heavy-tailed UCB algorithms based on concentration inequalities, the Median of Means UCB and the Truncated Mean UCB, have the lowest standard deviation of cumulative regret. However, as our previous results showed, their performance considering the average cumulative regret is poor, and the upper confidence bound generated by these algorithms are very conservative. These results also illustrate that for the other algorithms, in case of a small gap, $\Delta = 0.1$, the standard deviations are roughly the same, while for a larger gap, $\Delta = 0.5$, the RMM-UCB and MR-APE algorithms have the least standard deviation for the cumulative regret.
 \end{appendices}
\end{document}